\PassOptionsToPackage{table}{xcolor}
\documentclass[dvipsnames]{article}

\usepackage[T1]{fontenc}
\usepackage{graphicx}
\usepackage{wrapfig}
\usepackage{booktabs}
\usepackage{float}
\usepackage{multirow}
\usepackage{makecell}
\usepackage{array}
\usepackage{tabularx}
\usepackage{threeparttable}
\usepackage{amsmath,amssymb,amsthm}
\usepackage{mathtools}
\usepackage{bm}
\usepackage{algorithm}
\usepackage{algorithmic}
\usepackage{enumitem}
\usepackage{colortbl}
\usepackage{siunitx}
\usepackage{tcolorbox}
\usepackage{textcomp}
\usepackage[normalem]{ulem}
\usepackage{colm2024_conference}

\useunder{\uline}{\ul}{}
\tcbuselibrary{listings,breakable,skins}
\newtcblisting{kloopwideprompt}[1]{
  enhanced,
  unbreakable,
  listing only,
  width=\textwidth,
  title={#1},
  colback=black!2,
  colframe=black!25,
  colbacktitle=black!6,
  coltitle=black,
  fonttitle=\small\bfseries,
  boxrule=0.4pt,
  arc=1mm,
  left=2mm,
  right=2mm,
  top=1.5mm,
  bottom=1.5mm,
  before skip=8pt,
  after skip=10pt,
  listing options={
    language={},
    basicstyle=\ttfamily\fontsize{8.5}{10.5}\selectfont,
    keywordstyle=\color{black},
    commentstyle=\color{black},
    stringstyle=\color{black},
    breaklines=true,
    breakatwhitespace=false,
    columns=fullflexible,
    keepspaces=true,
    showstringspaces=false,
    showspaces=false,
    showtabs=false,
    mathescape=false,
    texcl=false,
    upquote=true,
    tabsize=2,
    aboveskip=0pt,
    belowskip=0pt
  }
}

\theoremstyle{definition}
\newtheorem{definition}{Definition}
\newtheorem{proposition}[definition]{Proposition}

\definecolor{headerGray}{HTML}{f0f0f0}
\definecolor{rowHighlight}{HTML}{ecf0ff}

\title{Weakness Localization Is Not Enough: State-Aware Recursive Knowledge Improvement for Large Language Models}

\author{
  \textbf{Runze Jiang\textsuperscript{1,2}},
  \textbf{Yan Wang\textsuperscript{2}},
  \textbf{Boran Jin\textsuperscript{2,3}},
  \\
  \textbf{Bingyu Zhu\textsuperscript{2,\textdagger}},
  \textbf{Longtao Huang\textsuperscript{2}}
  \\
  \textsuperscript{1}Peking University,
  \textsuperscript{2}Alibaba Group
  \\
  \textsuperscript{3}Delft University of Technology
  \\
  \small{\textbf{\textsuperscript{\textdagger}Corresponding author}}
}

\begin{document}

\maketitle

\begin{abstract}
Knowledge improvement in large language models is shifting from predefined knowledge incorporation pipelines toward evaluation-guided updating. Existing methods use evaluation to localize model-specific knowledge weaknesses and guide targeted post-training. However, identifying \emph{what to improve} does not necessarily provide sufficient information for deciding \emph{how to improve it}. The same localized weakness can be compatible with distinct operational knowledge states: the relevant knowledge may be unavailable, available but difficult to access reliably, or accessible but ineffective for downstream use. We call this ambiguity \textbf{knowledge-state aliasing}. When compatible states favor different interventions, greater diagnostic resolution alone is insufficient for intervention selection. Moreover, model updates can change the underlying knowledge state before the observed weakness disappears, such that an intervention appropriate in one round may no longer address the bottleneck in the next.
To address these challenges, we propose \textbf{Recursive Knowledge Improvement (RKI)}, which formulates knowledge improvement as a sequential, state-aware intervention process. RKI organizes knowledge into graph-structured units that provide a shared basis for diagnostic evaluation and targeted intervention across rounds. At each round, RKI localizes knowledge weaknesses, estimates the model's operational knowledge state from diagnostic evidence and intervention history, selects a corresponding intervention, and re-evaluates the updated model before determining the next intervention. On \textsc{HotpotEval} and \textsc{SafeInternalEval}, RKI outperforms existing knowledge incorporation methods by 18.41 points on average. Further analyses show that knowledge states provide additional guidance for intervention selection beyond weakness localization and can change across model updates even when the observed weakness persists.
\end{abstract}

\section{Introduction}

\begin{wrapfigure}{r}{0.48\textwidth}
  \vspace{-0.8em}
  \centering
  \includegraphics[width=\linewidth,height=0.24\textheight,keepaspectratio]{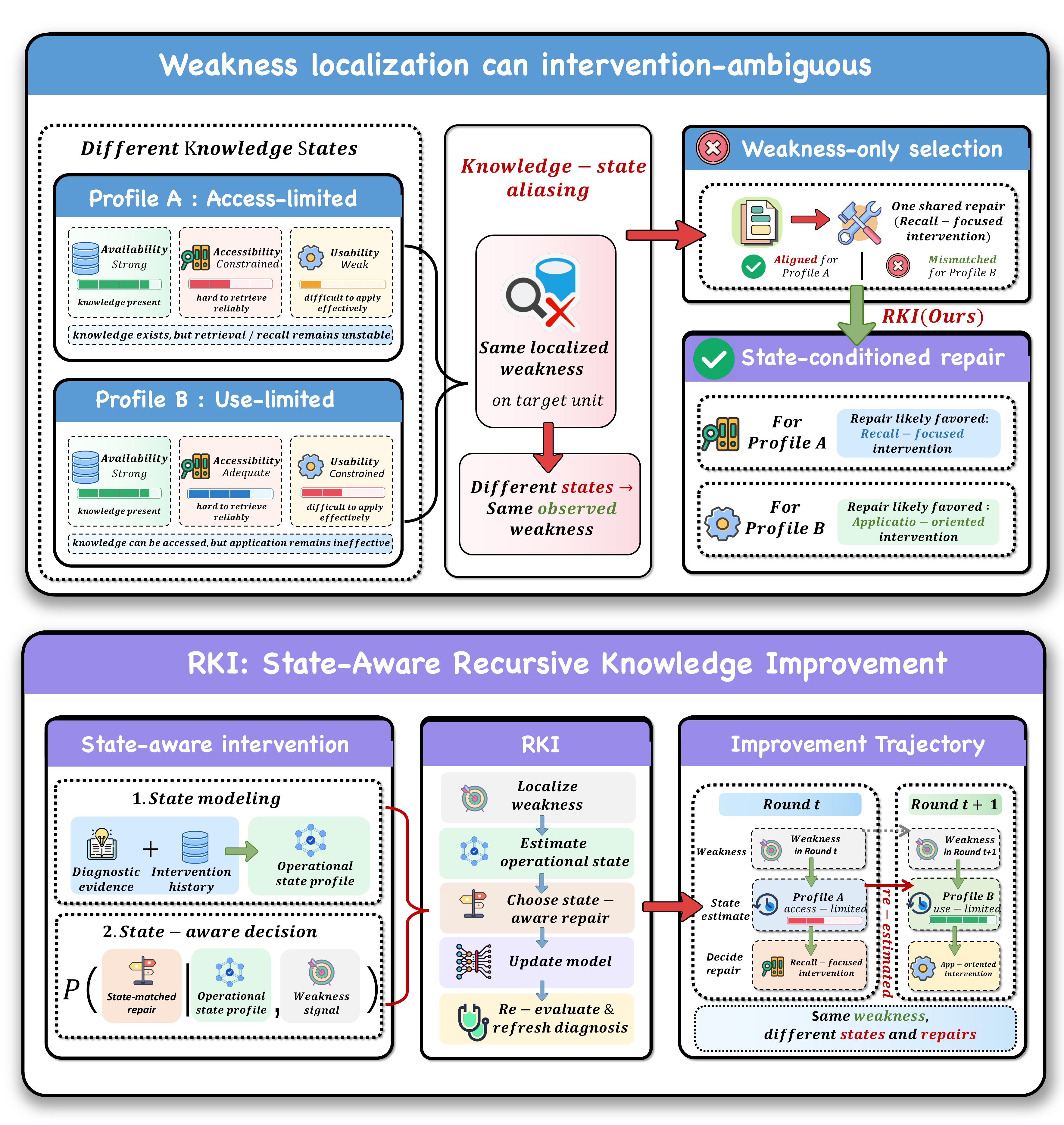}
  \caption{From weakness localization to state-aware knowledge improvement. RKI estimates operational knowledge states to guide interventions and re-estimates them after each model update.}
  \label{fig:motivation}
  \vspace{-0.8em}
\end{wrapfigure}

Knowledge improvement aims to incorporate external knowledge into large language models so that it can be reliably accessed and used in downstream tasks. In Web settings, this knowledge is typically derived from Web corpora and organized into knowledge graphs, which support applications such as multi-hop question answering and knowledge-grounded agents. Conventional incorporation pipelines~\citep{gururangan2020don,cheng2024adapting} typically follow a predefined procedure, where training data and update strategies are specified with limited consideration of the model's current knowledge state.

Recent work has shifted toward evaluation-guided improvement, where evaluation extends beyond retrospective assessment to serve as a signal for prioritizing subsequent knowledge updates. GraphGen and SENATOR identify model-specific knowledge deficiencies and target post-training toward the corresponding knowledge regions~\citep{chen2025graphgen,wei2026structural}. AutoDetect, EvalTree, and CRAFT further localize weaknesses and guide targeted improvement through agent-based assessment, hierarchical capability profiling, and rubric-level evaluation, respectively~\citep{cheng2024autodetect,zeng2025evaltree,gupta2026craft}. These methods identify \emph{what} should be improved at increasingly fine levels of granularity. However, even a fine-grained diagnosis may leave the choice among training interventions underdetermined. This raises a distinct question: \emph{does localizing a weakness provide sufficient information to determine how it should be improved?}

The information needed for this choice can lie in the model's \emph{operational knowledge state}, consistent with non-binary characterizations of LLM knowledge~\citep{li2025knowledge}. The same localized weakness can reflect different states: the relevant knowledge may be unavailable, available but difficult to access reliably, or accessible but ineffective for understanding, reasoning, or application. We call the compatibility of the same observed weakness with distinct operational states \textbf{knowledge-state aliasing}. This ambiguity becomes consequential for knowledge improvement when these states favor different interventions: factual exposure, recall-focused training, and application-oriented examples can address different bottlenecks within the same knowledge region. In such cases, greater \emph{diagnostic resolution} does not by itself ensure \emph{intervention sufficiency}. When compatible states share no optimal intervention, a weakness-only policy cannot be optimal for all of them. Figure~\ref{fig:motivation} illustrates this ambiguity and the state-aware intervention process that it motivates.

A separate challenge arises as knowledge states evolve across model updates. The same observed weakness may persist after an update even when the operational knowledge state underlying it has changed. For example, previously unavailable knowledge may become accessible but remain difficult to recall reliably, or become reliably recallable while remaining difficult to apply effectively. The appropriate intervention can therefore change before the observed weakness disappears. Continuing with the previous intervention may then fail to address the new bottleneck. This motivates re-evaluation and knowledge-state re-estimation before selecting the next intervention.

To address these issues, we propose \textbf{Recursive Knowledge Improvement (RKI)}, which views knowledge improvement as a sequential, state-aware intervention process. RKI organizes knowledge into graph-structured units that align diagnostic evidence, target knowledge, and intervention records across rounds. At each round, diagnostic evaluation first localizes the model's knowledge weaknesses, while diagnostic evidence and intervention history are used to estimate the model's operational knowledge-state profile. The weakness profile identifies \emph{what} needs improvement, while the state estimate provides additional information for deciding \emph{how} each weakness should be addressed. RKI then applies the selected intervention, re-evaluates the updated model over the same knowledge units, and revises both the weakness profile and state estimate before selecting the next intervention. This closed-loop process allows RKI to adapt subsequent interventions as the model's knowledge state evolves across rounds.

We evaluate RKI on the public multi-hop knowledge benchmark \textsc{HotpotEval} and the enterprise safety and compliance benchmark \textsc{SafeInternalEval}. RKI achieves the best performance on both benchmarks, improving the average score across them by 18.41 points over the strongest baseline. More importantly, our experiments identify two phenomena underlying these gains. First, the same localized failure can correspond to different patterns of knowledge access, and conditioning intervention selection on the operational knowledge state further improves RKI by 5.17 points over Weakness Only under the same weakness information. Second, even when an application weakness persists after a model update, 57.1\% of the corresponding knowledge units show a different recognition--recall pattern, showing that a persistent failure signal does not imply a stable underlying knowledge state. Together, these results support the two core design choices of RKI: using knowledge states to distinguish intervention needs under the same localized weakness, and re-estimating those states after model updates to guide subsequent interventions.

Our contributions are as follows:
\begin{itemize}
    \item We identify \textbf{knowledge-state aliasing} and characterize when diagnostic resolution is insufficient for intervention selection.

    \item We propose \textbf{Recursive Knowledge Improvement (RKI)}, which couples state-aware intervention selection with recursive knowledge-state re-estimation.

    \item We empirically show that operational knowledge states provide additional information for intervention selection beyond weakness localization, and that these states can change across model updates even when the observed weakness persists.
\end{itemize}

\section{Related Work}

\subsection{Knowledge Incorporation for Large Language Models}

Prior work has extensively explored incorporating external knowledge into large language models. Domain-specific models such as PMC-LLaMA, SaulLM, and HuatuoGPT-II acquire specialized knowledge through pre-training and instruction tuning~\citep{gururangan2020don,wu2023bloomberggpt,wu2024pmc,colombo2024saullm,chen2023huatuogpt}, while K-Adapter and knowledge-aligned domain shift tuning provide adapter-based and parameter-efficient adaptation, respectively~\citep{wang2021k,kawamae2025knowledge}. Other studies focus on constructing richer training data from a given knowledge source. Self-QA converts domain knowledge into instruction data~\citep{zhang2023self} and EntiGraph leverages entity relations to generate synthetic text for continued pre-training~\citep{yang2025synthetic}. AdaptLLM and Instruction Pre-Training likewise derive training data from raw corpora~\citep{cheng2024adapting,cheng2024instruction}. A complementary line edits parametric knowledge~\citep{mitchell2021fast,meng2022locating,meng2022mass,zhang2024instructedit,li2026emsedit} or post-processes black-box outputs~\citep{song2025assessing}, rather than recursively selecting training interventions. Despite these advances, knowledge incorporation is often organized as a predefined pipeline in which data construction and training strategies are specified in advance, with limited use of evaluation feedback to adapt subsequent knowledge updates.

\subsection{Evaluation-Guided Targeted Model Improvement}

Recent work has begun to use model evaluation to guide subsequent updates, making observed weaknesses a signal for targeted improvement. GraphGen uses calibration uncertainty over fine-grained knowledge graphs to identify knowledge gaps and prioritize synthetic data generation accordingly~\citep{chen2025graphgen}, while SENATOR combines structural entropy with MCTS to locate knowledge-deficient regions and generate targeted data for repair~\citep{wei2026structural}. Beyond knowledge-specific improvement, AutoDetect links agent-based weakness detection to targeted data augmentation~\citep{cheng2024autodetect}, while EvalTree and CRAFT use hierarchical or rubric-level evaluation for targeted fine-tuning~\citep{zeng2025evaltree,gupta2026craft}. More broadly, R-Select uses validation feedback for data selection~\citep{gao2026r}, while the weakness-guided methods above determine what should be improved, often at increasingly fine levels of diagnostic granularity. However, greater diagnostic resolution does not necessarily imply intervention sufficiency: localizing a weakness may still leave the choice among candidate interventions underdetermined.

\subsection{Knowledge States in Large Language Models}

Knowledge-boundary research provides taxonomies and discovery methods~\citep{li2025knowledge,wang2026knowledge}. More generally, knowledge in large language models is not always well described by a binary distinction between ``known'' and ``unknown''. \citet{kazemnejad2023measuring} distinguish knowledge acquisition from utilization, showing that acquired knowledge may still fail to support downstream tasks, while \citet{gekhman2024does} link fine-tuning on new knowledge to hallucinations. \citet{gekhman2025inside} distinguish internally encoded knowledge from knowledge expressed in model outputs, while~\citet{calderon2026empty} separate missing knowledge from different levels of recall accessibility. \citet{li2025demystifying} further distinguish task-relevant knowledge retrieval from subsequent reasoning. These studies show that model knowledge cannot be adequately characterized by answer correctness alone: relevant knowledge may differ in whether it has been acquired, can be reliably accessed, and can be effectively used in downstream reasoning or application. However, prior work primarily studies these distinctions as properties to be measured or diagnosed, with limited attention to their role in intervention selection. RKI instead treats the operational knowledge state as intervention-relevant information: it uses the estimated state to distinguish among candidate interventions for the same localized weakness, and re-estimates that state after each model update to adapt subsequent intervention decisions.

\begin{figure*}[t]
\centering
  \includegraphics[width=1\linewidth]{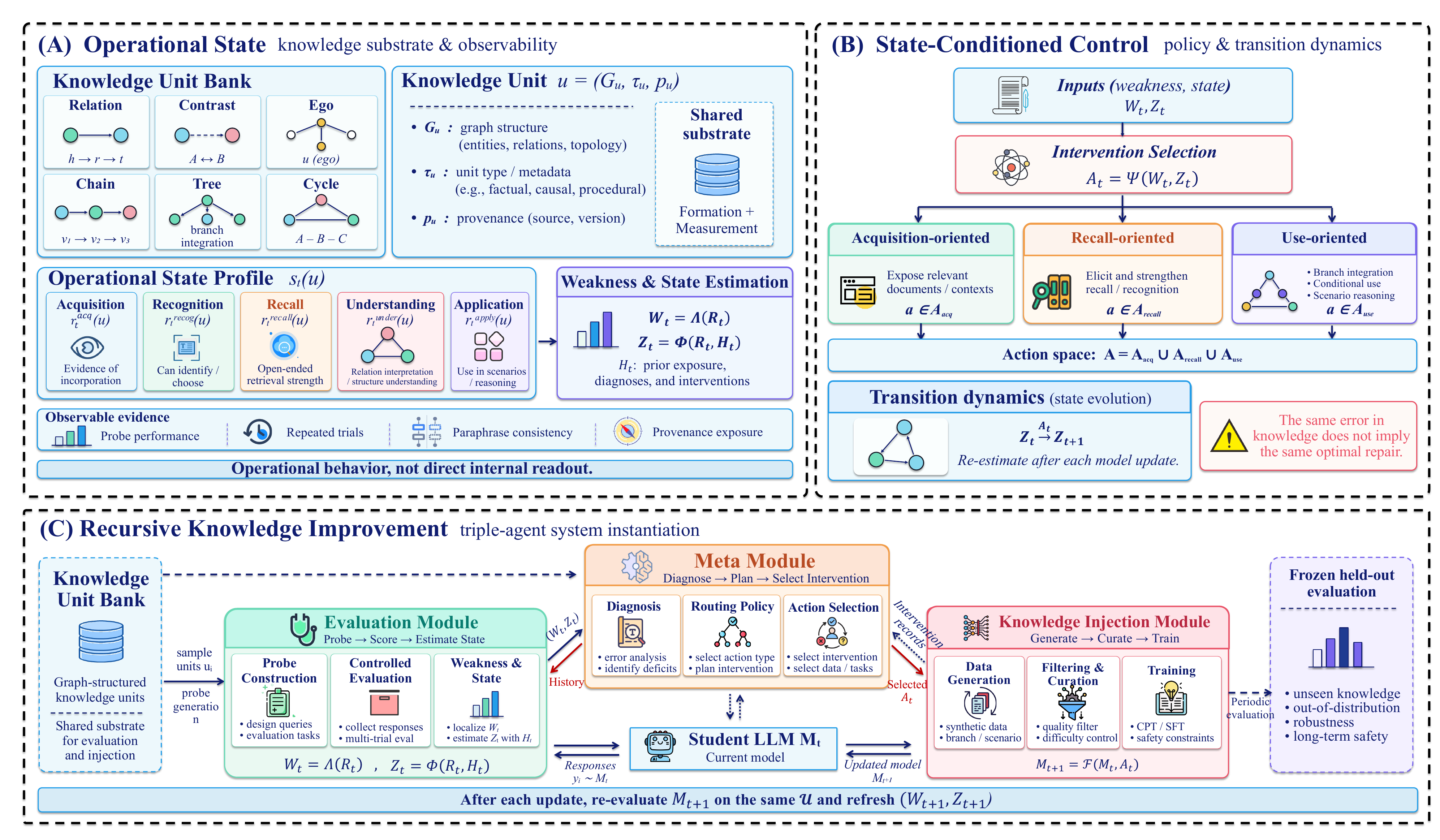}
  \caption{\textbf{Overview of RKI.} RKI aligns diagnosis and intervention over graph-structured knowledge units and recursively adapts interventions to the model's evolving operational knowledge state.}
  \label{fig:framework}
\end{figure*}

\section{Knowledge States for Knowledge Improvement}
\label{sec:knowledge_state_for_knowledge_improvement}

Given a domain knowledge base $\mathcal{K}$ and a model $M_t$ at round $t$, we consider a set of knowledge units $\mathcal{U}$ that provides a common basis for diagnostic evaluation and intervention-data generation. Diagnostic evaluation produces observations over these knowledge units, which we collectively denote by $R_t$. From $R_t$, we derive a weakness profile
\begin{equation}
W_t = \Lambda(R_t),
\label{eq:weakness_profile}
\end{equation}
where $W_t$ identifies the knowledge units on which the model exhibits deficiencies and thus localizes what currently requires improvement. We next introduce the operational knowledge state to characterize the knowledge conditions underlying these weaknesses and examine how knowledge interventions alter this state.

\subsection{Operational Knowledge State}

A weakness profile captures observable deficiencies revealed by diagnostic evaluation, but does not fully characterize the model's current knowledge condition. The model's \emph{operational knowledge state} at round $t$, denoted by $Z_t$, is a model-level profile estimated from the collective diagnostic evidence over $\mathcal{U}$. It summarizes the knowledge conditions relevant to subsequent intervention, including whether the relevant knowledge is available, can be reliably accessed, and can be effectively used for downstream understanding, reasoning, or application.

Because similar observable behavior may arise from different underlying knowledge conditions, the operational knowledge state cannot be identified from a single evaluation outcome alone. We abstract the state-estimation process as
\begin{equation}
Z_t = \Phi(R_t,H_t),
\label{eq:operational_knowledge_state}
\end{equation}
where $H_t$ summarizes the history available up to round $t$, including prior knowledge exposure, diagnostic results, and interventions, and $\Phi$ denotes the state-estimation function.

Thus, $W_t$ and $Z_t$ provide complementary information at different granularities. The weakness profile $W_t$ localizes the knowledge units that currently require improvement, while the model-level operational knowledge state $Z_t$ characterizes the model's overall knowledge condition and provides additional information for deciding how these targets should be addressed.

\subsection{Knowledge Intervention and State Transition}

Let $A_t$ denote the set of knowledge interventions applied at round $t$, which together specify the knowledge to be updated and the corresponding update strategy for each intervention target. Applying $A_t$ transforms the current model as
\begin{equation}
M_{t+1} = \mathcal{F}(M_t,A_t),
\label{eq:knowledge_intervention}
\end{equation}
where $\mathcal{F}$ denotes the model-update procedure.

A knowledge intervention may alter the model's operational knowledge state even when the observed weakness has not yet disappeared. For example, previously unavailable knowledge may become accessible after an update, while the limiting factor shifts from availability to reliable access, understanding, or application. Each model update therefore induces a corresponding state transition:
\begin{equation}
Z_t \xrightarrow{A_t} Z_{t+1}.
\label{eq:state_transition}
\end{equation}

This transition makes knowledge improvement inherently sequential. An intervention appropriate for $Z_t$ need not remain appropriate once the model reaches $Z_{t+1}$, even if the corresponding weakness remains observable. Continuing to rely on the previous state estimate after an update may therefore base subsequent interventions on stale information. Each intervention should consequently be followed by re-evaluation of the updated model and re-estimation of its current operational knowledge state.

\subsection{Knowledge-State Aliasing}

The preceding discussion raises a key question: does the weakness profile obtained from diagnostic evaluation contain sufficient information for intervention selection? When the same observed weakness is compatible with distinct operational knowledge states, conditioning on $W_t$ alone obscures state differences that may affect the appropriate intervention. We refer to this ambiguity as \emph{knowledge-state aliasing}.

For an observed weakness $w$, we define the set of compatible operational knowledge states as
\begin{equation}
\Gamma(w)
=
\left\{
\Phi(R,H)
\,|\,
\Lambda(R)=w
\right\}.
\label{eq:compatible_states}
\end{equation}

\begin{definition}[Knowledge-State Aliasing]
For an observed weakness $w$, \emph{knowledge-state aliasing} occurs if
\begin{equation}
|\Gamma(w)|>1,
\end{equation}
that is, if the same observed weakness is compatible with multiple distinct operational knowledge states.
\label{def:knowledge_state_aliasing}
\end{definition}

Knowledge-state aliasing alone does not imply different intervention requirements. It becomes consequential for intervention selection when states compatible with the same weakness have disjoint sets of optimal interventions. Let $\mathcal{A}$ denote the intervention space, and let $Q(z,a)$ denote the expected knowledge improvement from applying intervention $a\in\mathcal{A}$ in operational knowledge state $z$. The set of optimal interventions is
\begin{equation}
\mathcal{A}^{\star}(z)
=
\arg\max_{a\in\mathcal{A}} Q(z,a).
\label{eq:optimal_intervention_set}
\end{equation}

\begin{proposition}[Insufficiency of Weakness-Only Intervention]
Suppose an observed weakness $w$ admits two compatible operational knowledge states $z,z'\in\Gamma(w)$ such that
\begin{equation}
\mathcal{A}^{\star}(z)
\cap
\mathcal{A}^{\star}(z')
=
\varnothing.
\end{equation}
Then no intervention rule $\pi$ conditioned only on $w$ can be optimal for both $z$ and $z'$.
\label{prop:weakness_only_insufficient}
\end{proposition}

The result follows directly because a weakness-only rule assigns the same intervention to both states, while their optimal intervention sets are disjoint. It formalizes the gap between \emph{diagnostic resolution} and \emph{intervention sufficiency}: even a precisely localized weakness may remain insufficient for intervention selection when compatible operational knowledge states favor different interventions. Thus, the operational knowledge state provides intervention-relevant information beyond the weakness profile alone. A formal proof is provided in Appendix~\ref{app:proof_weakness_only}.
\section{Recursive Knowledge Improvement}
\label{sec:recursive_knowledge_improvement}

As established in Section~\ref{sec:knowledge_state_for_knowledge_improvement}, a localized weakness may be insufficient for intervention selection, while the operational knowledge state underlying that weakness can change after model updating. We instantiate these observations in \textbf{Recursive Knowledge Improvement (RKI)}, which alternates between diagnostic evaluation, state estimation, targeted intervention, and re-evaluation (Figure~\ref{fig:framework}).

\begin{table*}[t]
\centering
\caption{Statistics of the datasets used in our experiments. The corpus serves as the knowledge source for graph construction, diagnostic evaluation, and knowledge incorporation, while the test set is used for closed-book evaluation.}
\label{tab:benchmark_statistics}
\scriptsize
\setlength{\tabcolsep}{4.5pt}
\renewcommand{\arraystretch}{1.08}
\begin{tabular*}{\textwidth}{
@{\extracolsep{\fill}}
lll
cc
cc
cc
@{}}
\toprule
\multirow{2}{*}{\textbf{Dataset}}
& \multirow{2}{*}{\textbf{Domain}}
& \multirow{2}{*}{\textbf{Language}}
& \multicolumn{2}{c}{\textbf{Samples}}
& \multicolumn{2}{c}{\textbf{Avg. Tokens}}
& \multicolumn{2}{c}{\textbf{Max. Tokens}} \\
\cmidrule(lr){4-5}
\cmidrule(lr){6-7}
\cmidrule(lr){8-9}
& & & Corpus & Test & Corpus & Test & Corpus & Test \\
\midrule

\textsc{HotpotEval}
& General
& English
& 9,990
& 1,000
& 137.5
& 21.3
& 1,984
& 74 \\

\textsc{SafeInternalEval}
& Security \& Compliance
& Chinese
& 1,000
& 600
& 352.4
& 41.2
& 701
& 139 \\

\bottomrule
\end{tabular*}
\end{table*}

\subsection{Framework Overview}
\label{sec:rki_framework_overview}

Given a domain knowledge base $\mathcal{K}$ and a model $M_t$ at round $t$, RKI first constructs a knowledge graph $\mathcal{G}$ from $\mathcal{K}$ and extracts a set of graph-structured knowledge units $\mathcal{U}$. These units provide a common basis for diagnostic evaluation and intervention-data generation. Diagnostic evaluation produces observations $R_t$ over the knowledge units, from which RKI derives a weakness profile $W_t$ and, together with the intervention history $H_t$, estimates a model-level operational knowledge state $Z_t$. The weakness profile determines which knowledge units are intervention targets, while the global state profile informs the intervention strategy applied to those targets. RKI then assigns interventions to the identified targets, updates the model, and repeats the same process on the updated model:
\begin{equation}
\begin{aligned}
(M_t,\mathcal{K})
&\rightarrow R_t
\rightarrow (W_t,Z_t)
\rightarrow A_t \\
&
\rightarrow M_{t+1}
\rightarrow R_{t+1}
\rightarrow (W_{t+1},Z_{t+1})
\rightarrow \cdots .
\end{aligned}
\label{eq:rki_recursive_process}
\end{equation}

As summarized in Eq.~\eqref{eq:rki_recursive_process}, RKI does not treat a diagnosis as fixed across model updates. Each intervention is followed by re-evaluation, and subsequent intervention decisions are made from the updated weakness and state estimates.

\subsection{Graph-Structured Knowledge Units}

RKI organizes the knowledge in $\mathcal{G}$ into a collection of local graph structures
\begin{equation}
\mathcal{U}
=
\{u_1,u_2,\ldots,u_N\},
\label{eq:rki_knowledge_units}
\end{equation}
which serve as the common knowledge objects for diagnosis, intervention, and re-evaluation. We instantiate $\mathcal{U}$ in Eq.~\eqref{eq:rki_knowledge_units} with the six types summarized in Table~\ref{tab:rki_knowledge_units}. The same unit can be used to construct diagnostic probes, generate intervention data, and re-evaluate the updated model, providing a shared address space that aligns diagnosis and intervention across rounds.

\subsection{Operational Knowledge-State Estimation}
\label{sec:rki_operational_knowledge_state_estimation}

RKI characterizes the operational knowledge state in terms of three aspects: \emph{availability}, \emph{accessibility}, and \emph{usability}. To make these aspects operationally measurable, we instantiate them with six diagnostic dimensions: \emph{acquisition}, \emph{recognition}, \emph{recall}, \emph{understanding}, \emph{reasoning}, and \emph{application}, as summarized in Table~\ref{tab:rki_state_probes}.

Availability is estimated from acquisition-related diagnostic evidence, recognition and recall outcomes, together with the known intervention history that records which knowledge has been explicitly introduced in previous rounds. Accessibility is characterized primarily by recognition and recall, while usability is characterized by understanding, reasoning, and application. Together, these signals provide complementary evidence about the model's current operational knowledge condition.

At round $t$, let $R_t$ denote the diagnostic evidence collected over $\mathcal{U}$, and let $H_t$ summarize the history of the recursive improvement process up to round $t$, including known knowledge exposure, previous diagnostic results, state estimates, and intervention records. The operational knowledge-state profile is then estimated as
\begin{equation}
\begin{aligned}
Z_t
&=
\Phi(R_t,H_t) \\
&=
\left(
z_t^{\mathrm{acq}},
z_t^{\mathrm{recog}},
z_t^{\mathrm{recall}},
z_t^{\mathrm{under}},
z_t^{\mathrm{reason}},
z_t^{\mathrm{apply}}
\right).
\end{aligned}
\label{eq:rki_state_estimation}
\end{equation}

Each component in Eq.~\eqref{eq:rki_state_estimation} summarizes the corresponding diagnostic dimension over the current knowledge space.

\subsection{State-Aware Intervention and Recursive Re-estimation}
\label{sec:rki_state_aware_recursive_intervention}

Given $(W_t,Z_t)$, RKI treats each knowledge unit identified by $W_t$ as an intervention target. The target unit specifies which knowledge should be improved, while the model-level state $Z_t$ informs the intervention strategy applied to that target. For each weak knowledge unit $u\in W_t$, we write the selected intervention as
\begin{equation}
a_t(u)
=
\Psi(u,Z_t),
\qquad u\in W_t,
\label{eq:rki_state_aware_intervention}
\end{equation}
and collect the resulting interventions as
\begin{equation}
A_t
=
\{a_t(u)\mid u\in W_t\}.
\end{equation}
Here $\Psi$ denotes the intervention-selection procedure, which conditions the choice of strategy on the target unit and the model-level state.

We consider intervention families corresponding to different operational bottlenecks. Availability-oriented interventions increase exposure to the target knowledge, for example through continued pre-training or structured knowledge exposure. Accessibility-oriented interventions emphasize reliable retrieval, for example through recall-focused supervision. Usability-oriented interventions provide examples that require the target knowledge to be understood, reasoned over, or applied. These intervention families define a set of candidate strategies and do not impose a fixed mapping from knowledge states to training methods. The state estimate is used to identify which type of intervention is more appropriate under the current model condition.

Applying the selected intervention produces
\begin{equation}
M_{t+1}
=
\mathcal{F}(M_t,A_t),
\label{eq:rki_model_update}
\end{equation}
After the update in Eq.~\eqref{eq:rki_model_update}, RKI re-evaluates $M_{t+1}$ over the same knowledge units to obtain $R_{t+1}$ and recompute $(W_{t+1},Z_{t+1})$. The next intervention is therefore selected from the updated diagnosis. In this way, RKI adapts intervention decisions as the model's operational knowledge state evolves across rounds.

\section{Experiments}
\label{sec:experiments}

\subsection{Experimental Setup}
\label{sec:experimental_setup}

\paragraph{Benchmarks and Knowledge Sources.}
We evaluate RKI in two knowledge-intensive settings. For the public benchmark, we sample from \textsc{HotpotEval} released with GraphGen~\citep{chen2025graphgen}, retaining instances that can be aligned with their source corpus. \textsc{HotpotEval} is adapted from HotpotQA~\citep{yang2018hotpotqa} and contains general-domain multi-hop questions over multiple supporting facts. It provides a source corpus for knowledge incorporation and a QA test set for closed-book evaluation. We also construct \textsc{SafeInternalEval} from enterprise safety and compliance documents. Its evaluation questions are manually constructed and reviewed against the source documents, and are designed to avoid direct overlap with the knowledge units used for intervention-data generation. Dataset statistics are summarized in Table~\ref{tab:benchmark_statistics}, with additional details and examples provided in Appendix~\ref{app:benchmark_details_examples}.

\paragraph{Implementation.}
We use Qwen3-4B-Instruct-2507~\citep{yang2025qwen3} as the common backbone model and DeepSeek-V4-Pro~\citep{xu2026deepseek} for knowledge graph construction, evaluation-instance generation, and data filtering. From each source corpus, we construct a knowledge graph and extract graph-structured knowledge units that support diagnosis, weakness localization, and intervention-data generation. Table~\ref{tab:knowledge_unit_distribution} reports the 32,485 units extracted from \textsc{SafeInternalEval}. These counts reflect structural knowledge units in the graph and do not directly correspond to the number of evaluation or training instances. We further filter the units by question-generation suitability and knowledge value before constructing diagnostic and knowledge-improvement data.

\begin{table}[t]
\centering
\caption{Distribution of graph-structured knowledge units constructed from \textsc{SafeInternalEval}.}
\label{tab:knowledge_unit_distribution}
\small
\renewcommand{\arraystretch}{1.12}
\begin{tabular*}{\linewidth}{
@{\extracolsep{\fill}}
l
r
r
@{}}
\toprule
\textbf{Unit Type}
& \textbf{Units}
& \textbf{Proportion} \\
\midrule

Relation
& 14,439
& 44.45\% \\

Ego
& 1,340
& 4.12\% \\

Contrast
& 1,066
& 3.28\% \\

Cycle
& 975
& 3.00\% \\

Chain
& 13,494
& 41.54\% \\

Tree
& 1,171
& 3.60\% \\

\midrule
\textbf{Total}
& \textbf{32,485}
& \textbf{100.00\%} \\

\bottomrule
\end{tabular*}
\end{table}

\begin{figure*}
    \centering
    \includegraphics[width=1\linewidth]{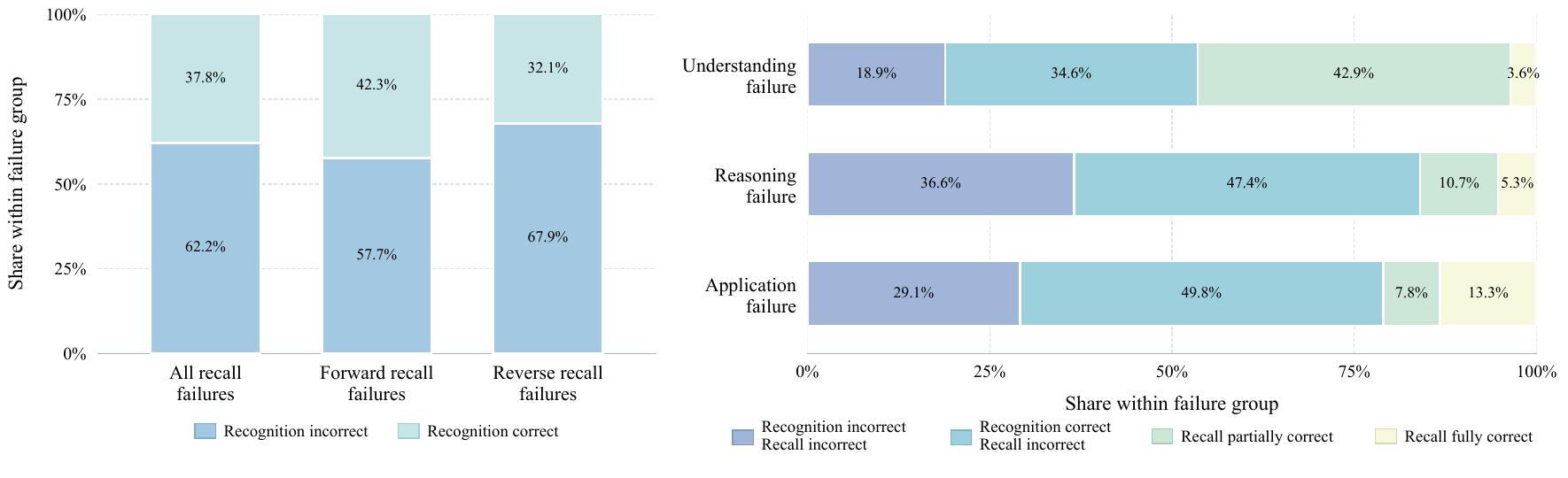}
    \caption{
Knowledge-state aliasing under the same localized failure signal.
Left: recognition outcomes for recall-failure instances.
Right: recognition and recall patterns for understanding-, reasoning-, and application-failure units.
}
    \label{fig:knowledge_state_aliasing}
\end{figure*}

\paragraph{Comparison Methods.}
Following the baseline setting of GraphGen~\citep{chen2025graphgen}, we compare RKI with GENIE~\citep{yehudai2024achieving}, LongForm~\citep{koksal2024longform}, WRAP~\citep{maini2024rephrasing}, Self-QA~\citep{zhang2023self}, EntiGraph~\citep{yang2025synthetic}, and GraphGen. All methods use the same source corpora, backbone model, and training configurations for a consistent comparison. Training details are provided in Appendix~\ref{app:training_configurations}. We do not include SENATOR~\citep{wei2026structural} in the main comparison because its deficiency detection assumes an external, curated domain knowledge graph, whereas our evaluation uses only the source corpus of each benchmark as the shared knowledge source. Applying SENATOR here would require replacing its external graph with an automatically constructed one, resulting in a substantially different experimental setting.

\paragraph{Evaluation.}
For single-choice recognition probes, we rotate answer positions across four variants and control option lengths to reduce position and length biases. For fill-in-the-blank recall probes, we perform three independent random samples to reduce sampling variance. For open-ended evaluation, we use rubric-based LLM judgments as the primary evaluation score. We do not use lexical-overlap metrics such as token-F1 or ROUGE-F as primary measures because they mainly capture surface-form similarity and can be sensitive to response length and paraphrasing, which may not reliably reflect semantic correctness in open-ended knowledge evaluation. The full judging protocol is provided in Appendix~\ref{app:judging_protocol}. We further validate the LLM judge against human annotations, with agreement results reported in Appendix~\ref{app:judging_agreement}. Inference settings are provided in Appendix~\ref{app:inference_configurations}.

\subsection{Knowledge-State Aliasing}

We first test whether the same localized weakness can correspond to different operational knowledge states. For instances sharing the same recall, understanding, reasoning, or application failure, we compare recognition and recall performance to examine whether a single weakness label conflates distinct patterns of knowledge access.

As shown in Figure~\ref{fig:knowledge_state_aliasing}, among instances with recall failure, 37.8\% still answer the corresponding recognition probe correctly, while 62.2\% also fail recognition. Similar heterogeneity appears for higher-order failures. Among units with application failure, 29.1\% fail both recognition and recall on the measured associated facts, while 13.3\% recall all of them. The corresponding proportions are 18.9\% and 3.6\% for understanding failure, and 36.6\% and 5.3\% for reasoning failure.

These results show that the same localized failure can correspond to different patterns of knowledge access. Recognition and recall evidence further separates operational knowledge-state profiles under the same failure signal, providing empirical support for \emph{knowledge-state aliasing}. Here, complete recall refers only to the measured associated facts and does not imply recall of all knowledge required by the higher-order task.

\subsection{Overall Knowledge Improvement}

As shown in Table~\ref{tab:overall_intervention_results}, RKI achieves the best performance on both knowledge-intensive benchmarks. On \textsc{HotpotEval}, RKI reaches an average score of 32.70, outperforming the strongest baseline, Self-QA, by 5.90 points. This result shows that RKI remains effective in a general-domain multi-hop setting where answering questions requires integrating multiple supporting facts.

The improvement is more pronounced on \textsc{SafeInternalEval}, where RKI achieves 58.92, compared with 28.00 for Self-QA and 22.33 for GraphGen. Averaged across the two benchmarks, RKI obtains 45.81, which is 18.41 points above the strongest baseline average of 27.40. Overall, these results show that RKI consistently improves knowledge incorporation across both public and enterprise-internal settings. The analyses that follow further examine whether these gains arise from state-aware intervention selection and recursive knowledge-state re-estimation.

\subsection{Recursive Knowledge-State Re-estimation}

\begin{table}[t]
\centering
\caption{Performance of different knowledge incorporation methods. Bold marks the best result.}
\label{tab:overall_intervention_results}
\small
\renewcommand{\arraystretch}{1.12}
\begin{tabular*}{\linewidth}{
@{\extracolsep{\fill}}
l
c
c
c
@{}}
\toprule
\textbf{Method}
& \textsc{HotpotEval}
& \textsc{SafeInternalEval}
& \textbf{Avg.} \\
\midrule

GENIE
& 23.00
& 9.75
& 16.38 \\

LongForm
& 24.40
& 11.17
& 17.79 \\

WRAP
& 22.90
& 9.75
& 16.33 \\

Self-QA
& 26.80
& 28.00
& 27.40 \\

EntiGraph
& 26.50
& 21.50
& 24.00 \\

GraphGen
& 23.70
& 22.33
& 23.02 \\

\textbf{RKI (Ours)}
& \textbf{32.70}
& \textbf{58.92}
& \textbf{45.81} \\

\bottomrule
\end{tabular*}
\end{table}

We further examine whether the operational knowledge state can change even when an observed weakness persists after a model update. We select knowledge units whose associated application questions remain failed both before and after the first intervention. As shown in Figure~\ref{fig:persistent_application_failure_state_transition}, 57.1\% of these units show a different recognition--recall pattern after the intervention despite the persistent application weakness. In particular, the proportion of units that fully recall the measured associated facts increases from 12.2\% to 40.9\%.

These results show that a persistent observed weakness does not imply a stable operational knowledge state. Reusing the previous intervention based only on the unchanged failure signal may therefore rely on stale state information after the model has been updated. This directly motivates recursive knowledge-state re-estimation: after each intervention, RKI re-evaluates the updated model, collects new diagnostic evidence, and revises its estimate of the current operational knowledge state before selecting the next intervention.

We next examine how the operational knowledge-state profile evolves across RKI updates. Figure~\ref{fig:rki_progression} reports the six diagnostic dimensions together with overall performance at each phase. Phase 0 denotes the base model, while the reference line marks the score of 50.50 obtained by direct training on the full intervention data. The overall score increases from 17.41 at Phase 0 to 31.50, 51.03, and 58.92 over successive updates. The final phase is 8.42 points above the full-data training reference, showing that recursive state-aware intervention yields stronger overall performance than direct full-data training in this setting.

By construction, acquisition reaches 1 after the first intervention because all target facts are included in the knowledge injection. This indicates complete exposure to the target knowledge and does not imply reliable recognition, recall, or use. The remaining dimensions therefore continue to evolve across subsequent updates. From Phase 2 to Phase 3, recall increases from 37.1 to 38.9, while understanding rises from 62.1 to 69.6 and application from 56.4 to 69.2. Thus, different dimensions of the operational knowledge state need not improve at the same rate. Together with the recognition--recall changes observed under persistent application weakness, these results support re-evaluating the model after each intervention, re-estimating its current operational knowledge state, and adapting subsequent interventions accordingly.

\begin{figure}[t]
  \centering
  \begin{minipage}[t]{0.48\linewidth}
    \centering
    \includegraphics[width=\linewidth,height=0.19\textheight,keepaspectratio]{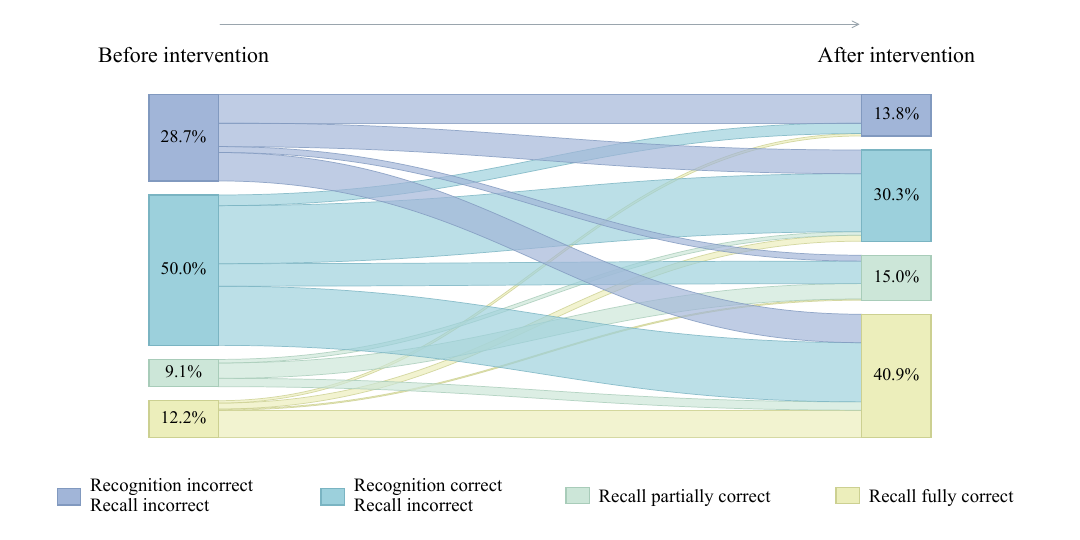}
    \caption{Operational knowledge-state transitions under a persistent application weakness.}
    \label{fig:persistent_application_failure_state_transition}
  \end{minipage}
  \hfill
  \begin{minipage}[t]{0.48\linewidth}
    \centering
    \includegraphics[width=\linewidth,height=0.19\textheight,keepaspectratio]{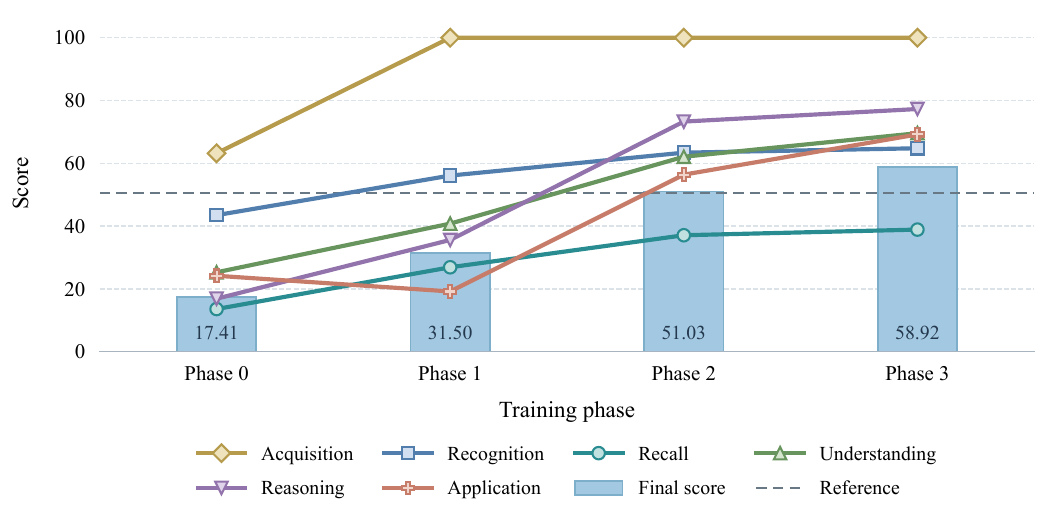}
    \caption{Evolution of the operational knowledge state and overall performance across recursive RKI updates.}
    \label{fig:rki_progression}
  \end{minipage}
\end{figure}

\subsection{Ablation and Analysis}

To isolate the contribution of state-aware intervention selection and recursive state re-estimation, we conduct an ablation study on \textsc{SafeInternalEval}. Random Intervention, Fixed Access, Fixed Utilization, Weakness Only, and RKI follow the same four-phase update schedule. Fixed Access and Fixed Utilization apply the same intervention type to all targeted knowledge units, while Weakness Only selects interventions based on the localized weakness $W_t$ without conditioning on the operational knowledge state $Z_t$. As an additional reference, Full-Data Training directly trains once on the complete intervention dataset without selective intervention assignment. Frozen State instead estimates the operational knowledge state once before the first intervention and reuses this estimate in all subsequent rounds, thereby removing recursive re-estimation while retaining state-aware intervention selection.

As shown in Table~\ref{tab:intervention_selection_ablation}, Fixed Utilization (48.00) outperforms Fixed Access (23.17), consistent with the larger share of utilization-oriented items in \textsc{SafeInternalEval}: improving access alone does not support understanding, reasoning, or application. Random Intervention reaches 29.33, remaining below the informed strategies.

Full-Data Training reaches 50.50, while Weakness Only reaches 53.75, exceeding both full-data training and the strongest fixed policy. Weakness localization therefore helps assign intervention data, and simply using all available data is insufficient.

Frozen State drops to 39.00, below both RKI and Weakness Only. Stale state estimates can therefore be less useful than ignoring the state, directly supporting recursive re-estimation after each update.

RKI achieves 58.92, 5.17 points above Weakness Only and 8.42 above Full-Data Training. Since both RKI and Weakness Only use the same localized weakness, this gain supports the additional value of state-aware intervention selection.

\begin{table}[t]
\centering
\caption{Ablation study on intervention selection on \textsc{SafeInternalEval}. Bold marks the best result.}
\label{tab:intervention_selection_ablation}
\small
\renewcommand{\arraystretch}{1.12}
\begin{tabular}{l c}
\toprule
\textbf{Method} & \textbf{Average Score} \\
\midrule

Random Intervention & 29.33 \\
Fixed Access & 23.17 \\
Fixed Utilization & 48.00 \\
Full-Data Training & 50.50 \\
Weakness Only & 53.75 \\
Frozen State & 39.00 \\
\textbf{RKI (Ours)} & \textbf{58.92} \\

\bottomrule
\end{tabular}
\end{table}

\section{Conclusion}

This work examines a limitation of evaluation-guided knowledge improvement: localizing a model's knowledge weaknesses does not necessarily provide sufficient information for deciding how those weaknesses should be addressed. We characterize this ambiguity as \emph{knowledge-state aliasing}, where the same observed weakness can correspond to different operational knowledge states, and show that weakness-only intervention selection can be insufficient when these states favor different interventions. Building on this observation, we propose Recursive Knowledge Improvement (RKI), which uses graph-structured knowledge units to align diagnosis, intervention, and re-evaluation, selects interventions based on the model's current operational knowledge state, and re-estimates that state after each update.

On \textsc{HotpotEval} and \textsc{SafeInternalEval}, RKI improves the average score across the two benchmarks by 18.41 points over the strongest baseline. Further analyses show that operational knowledge states provide information for intervention selection beyond weakness localization, and that these states can change after model updates even when the observed weakness persists. Taken together, these results suggest that knowledge improvement is better treated as a closed-loop process, with interventions continuously adapted to the model's evolving knowledge state.
\section{Limitation}

This work has two limitations. First, the operational knowledge state is estimated from diagnostic evidence and intervention history. It is intended as an actionable representation for intervention decisions, not as a recovery of the model's true latent knowledge state. The current state representation should therefore be viewed as a diagnostic abstraction for knowledge improvement, whose effectiveness depends on the chosen diagnostic dimensions and estimation procedure. Second, RKI currently performs state-aware intervention selection over a predefined set of candidate interventions and re-evaluates the model after each update to re-estimate its knowledge state, which introduces additional evaluation and training costs. Future work could explore more flexible state representations and intervention spaces, as well as more efficient recursive knowledge improvement while preserving intervention effectiveness.

\clearpage
\begingroup
\setlength{\bibsep}{0pt}
\bibliographystyle{colm2024_conference}
\bibliography{main}
\endgroup

\clearpage
\appendix

\section{Proof of Proposition~\ref{prop:weakness_only_insufficient}}
\label{app:proof_weakness_only}

\begin{proof}
Suppose, for contradiction, that there exists an intervention rule
$\pi:\mathcal{W}\rightarrow\mathcal{A}$ conditioned only on the observed weakness
that is optimal for both $z$ and $z'$. Unlike the state-aware selector $\Psi$ in
Eq.~\eqref{eq:rki_state_aware_intervention}, such a rule does not condition on the
model-level state. Since $z,z'\in\Gamma(w)$ correspond to the
same observed weakness $w$, the rule selects the same intervention
\begin{equation}
a_w=\pi(w)
\end{equation}
for both states.

For $\pi$ to be optimal in state $z$, we must have
\begin{equation}
a_w\in\mathcal{A}^{\star}(z).
\end{equation}
Similarly, optimality in state $z'$ requires
\begin{equation}
a_w\in\mathcal{A}^{\star}(z').
\end{equation}
Therefore,
\begin{equation}
a_w
\in
\mathcal{A}^{\star}(z)
\cap
\mathcal{A}^{\star}(z').
\end{equation}

However, by the assumption of Proposition~\ref{prop:weakness_only_insufficient},
\begin{equation}
\mathcal{A}^{\star}(z)
\cap
\mathcal{A}^{\star}(z')
=
\varnothing,
\end{equation}
which is a contradiction. Hence, no intervention rule conditioned only on $w$
can be optimal for both $z$ and $z'$.
\end{proof}

\section{RKI Knowledge Units and Diagnostic Dimensions}
\label{app:rki_units_and_dimensions}

We provide detailed definitions of the knowledge units and diagnostic dimensions introduced in Section~\ref{sec:recursive_knowledge_improvement}.

\begin{table}[t]
\centering
\caption{Graph-structured knowledge units in RKI and their functional roles.}
\label{tab:rki_knowledge_units}
\small
\renewcommand{\arraystretch}{1.2}
\begin{tabular*}{\linewidth}{@{\extracolsep{\fill}}lp{0.78\linewidth}@{}}
\toprule
\textbf{Unit} & \textbf{Role} \\
\midrule
Relation
& Represents a direct relational fact between two entities as the basic unit of relational knowledge. \\
Contrast
& Organizes related facts or local structures along a shared comparison dimension. \\
Ego
& Represents the local relational neighborhood centered on a focal entity. \\
Cycle
& Represents a closed relational path formed by multiple entities and their relations. \\
Chain
& Represents a multi-hop relational path connecting a sequence of entities. \\
Tree
& Represents a branching relational structure extending from a shared local context. \\
\bottomrule
\end{tabular*}
\end{table}
\begin{table}[t]
\centering
\caption{Diagnostic dimensions, operational aspects, and interpretations for knowledge-state estimation.}
\label{tab:rki_state_probes}
\small
\renewcommand{\arraystretch}{1.2}
\begin{tabular*}{\linewidth}{
@{\extracolsep{\fill}}
l
l
p{0.51\linewidth}
@{}}
\toprule
\textbf{Dimension} & \textbf{Aspect} & \textbf{Interpretation} \\
\midrule
Acquisition
& Availability
& Whether the relevant knowledge shows evidence of having been incorporated into the model. \\

Recognition
& Accessibility
& Whether the model can identify the relevant knowledge when candidate information or relevant cues are provided. \\

Recall
& Accessibility
& Whether the model can reliably retrieve the relevant knowledge without explicit candidate cues. \\

Understanding
& Usability
& Whether the model can correctly interpret the relations and meaning expressed by the relevant knowledge. \\

Reasoning
& Usability
& Whether the model can combine and derive information from the relevant knowledge through reasoning. \\

Application
& Usability
& Whether the model can effectively use the relevant knowledge in downstream tasks or problem settings. \\
\bottomrule
\end{tabular*}
\end{table}

\section{Prompt Templates}
\label{app:prompt_templates}

Due to space constraints, we present only the knowledge graph extraction and relation fact extraction prompts in Figures~\ref{fig:prompt_kg_extraction} and~\ref{fig:prompt_relation_fact}. Our pipeline additionally uses prompts for iterative graph extraction, relation fact extraction and filtering, diagnostic benchmark generation and filtering, and CPT- and SFT-data generation. The complete set of prompt templates will be made available in the accompanying code release.

\begin{figure*}[!tp]
  \centering
\begin{kloopwideprompt}{P1. Knowledge Graph Extraction (abridged)}
You are an NLP expert, skilled at analyzing text to extract named entities and their
relationships. Use English as the output language.

# Source Text
${input_text}

# Entity Types
${entity_types}

# Steps
1. Identify all entities. Judge from the source text only; do NOT add facts from outside
   the source. Keep only objects with a distinct identity or stable meaning that can be
   reused in later knowledge organization.

   For each entity, extract:
   - name: the most complete and stable name used in the source
   - type: one of the predefined entity types
   - description: identity, attributes, or roles explicitly expressed in the source

2. Identify relationships between entities. Find entity pairs that are clearly related and
   explicitly supported by the source.

   For each relation, extract:
   - source
   - target
   - description: the concrete semantic connection between the two entities

3. Output one complete XML document. Precision is preferred over recall; do not add
   unsupported entities or relations.

# Output Format
<knowledge_graph>
  <entities>
    <entity>
      <name>entity name</name>
      <type>entity type</type>
      <description>source-grounded description</description>
    </entity>
  </entities>
  <relationships>
    <relationship>
      <source>source entity name</source>
      <target>target entity name</target>
      <description>concrete relation</description>
    </relationship>
  </relationships>
</knowledge_graph>
\end{kloopwideprompt}
  \caption{Knowledge graph extraction prompt.}
  \label{fig:prompt_kg_extraction}
\end{figure*}

The extraction process additionally uses continuation and termination
prompts. The continuation prompt requests only entities and relations
that remain missing from previous rounds, while the termination prompt
determines whether another extraction round is necessary.

\medskip
\noindent
The relation fact extraction prompt first evaluates whether a relation
is worth testing, then determines which querying directions admit a
unique answer.

\begin{figure*}[!tp]
  \centering
\begin{kloopwideprompt}{P2. Relation Fact Extraction}
You are an expert constructing knowledge evaluation benchmarks, working through two
ordered steps: value judgment first, then direction judgment.

# Relation Pair
${relation_block}

# Source Material
${raw_context_block}

# Step 1: Value Judgment
Judge from the source material alone whether the relation between the two entities is
worth testing.

Set valuable to false if:
- the relation is unsupported or semantically indeterminate;
- it contains multiple independent propositions;
- it represents low-value metadata, transient states, or trivia;
- no direction can yield a question with a unique answer.

# Step 2: Direction Judgment
Only when valuable is true, determine the subject/object direction and whether each
direction admits a uniquely answerable question.

Choose exactly one:
- both
- forward
- reverse
- none

# Output Format
If the value judgment fails:
{"valuable": false, "valuable_reason": "..."}

Otherwise:
{"valuable": true,
 "valuable_reason": "...",
 "subject": "...",
 "object": "...",
 "direction": "both/forward/reverse/none",
 "direction_reason": "..."}
\end{kloopwideprompt}
  \caption{Relation fact extraction prompt (P2, abridged). Value judgment precedes the assessment of valid querying directions.}
  \label{fig:prompt_relation_fact}
\end{figure*}

\section{Knowledge Graph Visualization}

Figure~\ref{fig:hotpot_graph_visualization} illustrates the local structure of the knowledge graph constructed from the Hotpot corpus. Panel (a) shows the complete two-hop neighborhood centered on \textit{Star Wars}, containing 130 entities and 162 edges. Panel (b) provides a detailed view of 15 selected entities and all 25 edges among them. Node colors and shapes distinguish entity types, illustrating connections among works, people or fictional characters, and organizations.

\begin{figure}
    \centering
    \includegraphics[width=1\linewidth]{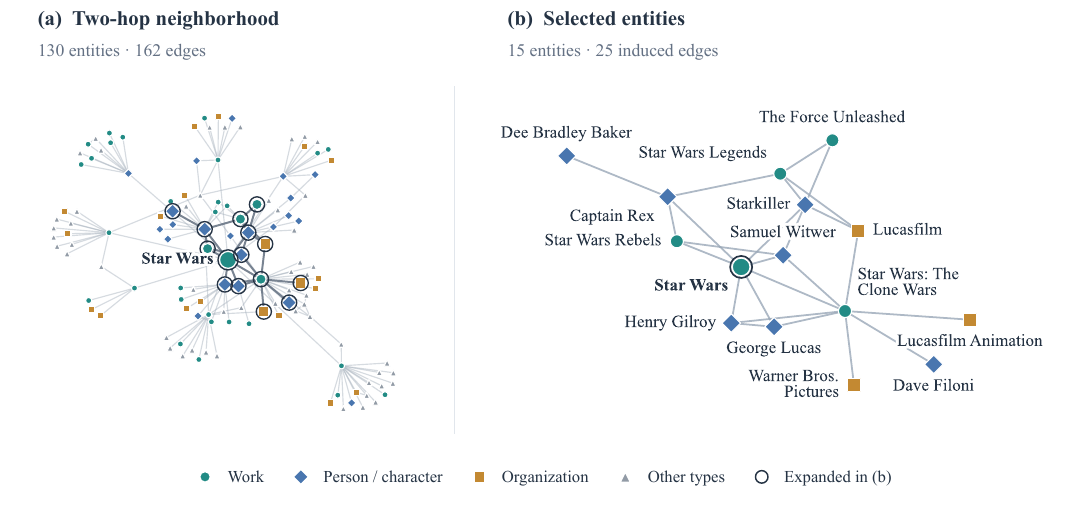}
    \caption{Local structure of the knowledge graph constructed from the Hotpot corpus. (a) The complete two-hop neighborhood of \textit{Star Wars}, with 130 entities and 162 edges. Outlined nodes are expanded in (b). (b) The induced subgraph of 15 selected entities and their 25 edges. Colors and markers denote entity types. All edges are undirected.}
    \label{fig:hotpot_graph_visualization}
\end{figure}

\section{Supplementary Experimental Details}
\label{app:appendix_experimental_details}

\subsection{Training Configurations}
\label{app:training_configurations}

We use separate training configurations for CPT and SFT, as summarized in Table~\ref{tab:training_configurations}.

\begin{table}[t]
\centering
\caption{Training configurations for CPT and SFT.}
\label{tab:training_configurations}
\small
\renewcommand{\arraystretch}{1.15}
\begin{tabular*}{\linewidth}{
@{\extracolsep{\fill}}
l
c
c
@{}}
\toprule
\textbf{Parameter}
& \textbf{CPT}
& \textbf{SFT} \\
\midrule

Context Length
& 2,048
& 2,048 \\

Global Batch Size
& 16
& 64 \\

Epochs
& 1
& 1 \\

Peak Learning Rate
& $5\times10^{-6}$
& $2\times10^{-5}$ \\

LR Scheduler
& Cosine
& Linear \\

Warmup Ratio
& 0.05
& 0.03 \\

Weight Decay
& 0.01
& 0.1 \\

Optimizer
& AdamW
& AdamW \\

Gradient Clipping
& 1.0
& 1.0 \\

Precision
& BF16
& BF16 \\

\bottomrule
\end{tabular*}
\end{table}

\subsection{Inference Configurations}
\label{app:inference_configurations}

We use the same inference configuration across all models and benchmarks, with greedy decoding, a maximum input length of 5,120 tokens, a maximum generation length of 2,048 tokens, and thinking mode disabled.

\subsection{Benchmark Details and Examples}
\label{app:benchmark_details_examples}

Each benchmark consists of a source corpus used for graph construction, diagnostic evaluation, and knowledge incorporation, together with a QA set for closed-book evaluation. For the public benchmark, we sample from the dataset released with GraphGen and retain test instances that can be aligned with their corresponding source texts.

\textsc{HotpotEval} is derived from HotpotQA~\citep{yang2018hotpotqa}, where answering a question requires combining evidence distributed across multiple supporting passages. We use the corresponding supporting texts as the source corpus and retain aligned questions for evaluation.

Table~\ref{tab:benchmark_examples} provides a representative corpus--test pair from \textsc{HotpotEval}. We additionally evaluate on \textsc{SafeInternalEval}, which is constructed from enterprise-internal safety and compliance documents. Its evaluation questions are manually constructed and reviewed against the source documents, and are designed to avoid direct overlap with the knowledge units used for intervention-data generation. Because the underlying documents contain proprietary information, we do not release or display examples from this benchmark.

\begin{table*}[t]
\centering
\caption{Representative source-corpus and evaluation examples from the public benchmarks.}
\label{tab:benchmark_examples}
\scriptsize
\renewcommand{\arraystretch}{1.15}
\begin{tabular*}{\textwidth}{
@{\extracolsep{\fill}}
p{0.12\textwidth}
p{0.40\textwidth}
p{0.40\textwidth}
@{}}
\toprule
\textbf{Dataset}
& \textbf{Corpus}
& \textbf{Test} \\
\midrule

\textsc{HotpotEval}
&
Bolax is a genus of flowering plants in the Apiaceae, with four or five species, and is endemic to temperate South America.

Jovibarba heuffelii, commonly known as Hen-and-chickens, is a plant species native to the Balkans and the Carpathians in Europe. \ldots
&
\textbf{Question:} Which plant has more species, Bolax or Jovibarba?

\textbf{Answer:} Bolax. \\

\bottomrule
\end{tabular*}
\end{table*}

\subsection{Evaluation and Judging Protocol}
\label{app:judging_protocol}

For open-ended evaluation, we use DeepSeek-V4-Pro as the LLM judge with temperature set to 0. Each evaluation dimension is scored on a 0--10 scale, with a score of at least 6 considered a pass. We use benchmark-specific evaluation dimensions, as summarized in Table~\ref{tab:judging_rubric}. An instance passes the strict criterion only when all dimensions receive scores of at least 6.

\begin{table*}[t]
\centering
\caption{Judging rubrics across benchmarks. Each dimension is scored from 0 to 10, with scores of at least 6 considered passing. For the internal benchmark, understanding evaluates association coverage, while reasoning and application evaluate reasoning coverage.}
\label{tab:judging_rubric}
\small
\renewcommand{\arraystretch}{1.15}
\begin{tabularx}{\textwidth}{
@{}
p{0.235\textwidth}
l
X
@{}}
\toprule
\textbf{Benchmark}
& \textbf{Dimension}
& \textbf{Evaluation Criterion} \\
\midrule

\textsc{HotpotEval}
& \texttt{final\_fact}
& Whether the final answer states the same fact as the reference answer. \\

& \texttt{completeness}
& Whether the answer covers the required constituents and qualifiers of the
  reference answer. \\

\midrule

\textsc{SafeInternalEval}\newline\textit{Understanding}
& \texttt{factual\_accuracy}
& Whether every fact stated in the answer is correct; hallucinated facts lose
  points. \\

& \texttt{association\_coverage}
& How many of the required answer points the answer covers. Paraphrase and
  superordinate summaries count as coverage; a point mentioned with a wrong
  association does not. \\

& \texttt{conclusion\_correctness}
& Whether the core judgment (overall positioning, the difference identified, or
  the misconception clarified) matches the reference, with the correct
  characterisation basis. \\

\midrule

\textsc{SafeInternalEval}\newline\textit{Reasoning and Application}
& \texttt{factual\_accuracy}
& Whether every fact stated in the answer is correct; hallucinated facts lose
  points. \\

& \texttt{reasoning\_coverage}
& How many of the required answer points the answer covers. Coverage must be
  substantively correct: naming a point while stating its derivation wrongly does
  not count. \\

& \texttt{conclusion\_correctness}
& Whether the final conclusion matches the reference, with the correct
  characterisation basis rather than one derivable from general knowledge alone. \\

\bottomrule
\end{tabularx}
\end{table*}

\subsection{Agreement Between LLM Judgments and Human Annotations}
\label{app:judging_agreement}

\begin{table}[!htbp]
\centering
\caption{Agreement between LLM judgments and human annotations on 100 stratified instances per benchmark.}
\label{tab:llm_human_agreement}
\small
\renewcommand{\arraystretch}{1.15}
\begin{tabular*}{\linewidth}{
@{\extracolsep{\fill}}
l
c
c
@{}}
\toprule
\textbf{Benchmark}
& \textbf{Pass Agreement}
& \textbf{Cohen's $\kappa$} \\

\midrule

\textsc{HotpotEval}
& 98\%
& $0.95$ \\

\textsc{SafeInternalEval}
& 96\%
& $0.89$ \\

\bottomrule
\end{tabular*}
\end{table}

To validate the LLM judge, we sample 100 instances from each benchmark, stratified by automated evaluation outcomes. Human annotators independently evaluate the sampled instances using the same rubric and pass threshold without access to the LLM judgments. An instance is considered to pass only when all dimensions receive scores of at least 6.

As shown in Table~\ref{tab:llm_human_agreement}, LLM and human judgments achieve pass agreement of 98\% and 96\% on \textsc{HotpotEval} and \textsc{SafeInternalEval}, with Cohen's $\kappa$ of $0.95$ and $0.89$, respectively. These results indicate consistently high agreement between automatic and human evaluation across both benchmarks.

\end{document}